\documentclass{IOS-Book-Article}

\usepackage{mathptmx}
\usepackage{soul}\setuldepth{article}
\def\hb{\hbox to 11.5 cm{}}

\usepackage{latexsym}
\usepackage{amssymb}
\usepackage{amsmath}
\usepackage{amsthm}
\usepackage{booktabs}
\usepackage{enumitem}
\usepackage{graphicx}
\usepackage{color}

\usepackage{xcolor}
\usepackage{tree-dvips}
\usepackage{qtree}
 \usepackage{float}

\usepackage{stmaryrd}
\usepackage{wrapfig}
\usepackage{cutwin}
\usepackage{lipsum}
 \usepackage{orcidlink}
 
\newtheorem{theorem}{Theorem}

\newtheorem{proposition}[theorem]{Proposition}

\newtheorem{definition}{Definition}

\newenvironment{varitemize}
{
\begin{list}{\labelitemi}
{\setlength{\itemsep}{0pt}
 \setlength{\topsep}{1.5pt}
 \setlength{\parsep}{0pt}
 \setlength{\partopsep}{0pt}
 \setlength{\leftmargin}{8.5 pt}
 \setlength{\rightmargin}{0pt}
 \setlength{\itemindent}{0pt}
 \setlength{\labelsep}{5pt}
 \setlength{\labelwidth}{10pt}
}}
{
 \end{list}
}

\newcommand{\BibTeX}{B\kern-.05em{\sc i\kern-.025em b}\kern-.08em\TeX}

\newcommand{\defin}{=_{\textit{def}} }

\newcommand{\atm}{\mathit{Atm}}

\newcommand{\val}{\mathit{Val}}

\renewcommand{\phi}{\varphi}

\newcommand{\putaway}[1]{}

\newcommand{\plt}{\mathit{Plt}}
\newcommand{\dfd}{\mathit{Dfd}}

\newcommand{\relevance}{\mathcal{R}}

\newcommand{\relHierarchy}{\mathcal{H}}
\newcommand{\relBinding}{\mathcal{B}}

\newcommand{\temporalorder}{\leq_{T}}
\newcommand{\stricttemporalorder}{<_{T}}

\newcommand{\court}{\mathit{c}}

\newcommand{\Facts}{\mathtt{Facts}}
\newcommand{\Courts}{\mathtt{Courts}}
\newcommand{\Org}{\mathit{Org}}

\newcommand{\outcome}{\mathit{o}}

\newtheorem{example}{Example}

\begin{document}

\pagestyle{headings}
\def\thepage{}
\begin{frontmatter}              % The preamble begins here.

%\pretitle{Pretitle}
\title{When Precedents Clash}
\runningtitle{When Precedents Clash }
%\markboth{}{September 2024\hb}
%\subtitle{Subtitle}

\author[A,C]{\fnms{Cecilia} \snm{Di Florio}\orcidlink{0000-0002-8927-7414}%
\thanks{Corresponding Author: Cecilia Di Florio, cecilia.diflorio2@unibo.it}},
\author[B]{\fnms{Huimin} \snm{Dong}\orcidlink{0000-0002-4951-0111}} 
and
\author[A]{\fnms{Antonino} \snm{Rotolo}\orcidlink{0000-0001-5265-0660}}%{

\runningauthor{C. Di Florio, H. Dong and A. Rotolo}
\address[A]{University of Bologna, \textsuperscript{b }TU Wien, \textsuperscript{c}University \:of\: Luxembourg}  
%\address[B]{TU Wien}

\begin{abstract}
Consistency of case bases is a way to avoid the problem of retrieving conflicting constraining precedents for new cases to be decided. However, in legal practice the consistency requirements for cases bases may not be satisfied. As pointed out in \cite{Broughton}, a model of precedential constraint should take into account the hierarchical structure of the specific legal system under consideration and the temporal dimension of cases. This article continues the research initiated in \cite{liu2022modelling, lail}, which established a connection between Boolean classifiers and legal case-based reasoning. On this basis, we enrich the classifier models with an organisational structure that takes into account both the hierarchy of courts and which courts issue decisions that are binding/constraining on subsequent cases. We focus on common law systems. We also introduce a temporal relation between cases. Within this enriched framework, we can formalise the notions of overruled cases and cases decided \emph{per incuriam}: such cases are not to be considered binding on later cases. Finally, we show under which condition principles based on the hierarchical structure and on the temporal dimension can provide an unambiguous decision-making process for new cases in the presence of conflicting binding precedents.
\end{abstract}

\end{frontmatter}
%\markboth{September 2024\hb}{September 2024\hb}
%\thispagestyle{empty}
%\pagestyle{empty}

%%%%%%%%%%%%%%%%%%%%%%%%%%%%%%%%%%%%%%%%%%%%%%%%%%%%%%%%%%%%%%%%%%%%%%%%
\vspace{-15pt}
\section{Introduction}

The use of machine  learning (ML) to predict the outcomes of legal procedures is widely discussed in the literature, by legal experts and by policymakers (see, e.g., \cite{Gan_Kuang_Yang_Wu_2021,Medvedeva2020-MEDUML,BexP21,ATKINSON2020103387}). Specifically, some concerns are raised by  judges regarding the use of AI in the courts. Firstly, AI can undermine the independent exercise of judicial power. Moreover, it is not established at all that the outcomes of judicial algorithmic decision predictors are normatively correct, accurate and robust. Finally, ML decision making is far from  being transparent and explainable.
To address these concerns, symbolic methods are needed to formally verify the robustness of machine learning algorithms used in predictive justice.  ML algorithms in predictive justice provide the outcome of new cases  on the basis of previous cases; in this sense they perform \emph{case base reasoning} (CBR). And indeed, in \cite{liu2022modelling}, a correspondence between formal  models of legal CBR and binary input classifier models (CMs) \cite{LiuLoriniJLC, LiuLorini2021BCL} has been established. In particular, the correspondence had been identified between specific CMs and the \emph{reason} and \emph{result model} proposed by Horty \cite{Horty2011RR}. Horty's models aim to formalise a founding concept of the common law, that of  \emph{precedential constraint}, by  defining the conditions whereby the decision of a new case is forced by a body of previous cases.  However, in these models the consistency of the starting case base is assumed, i.e. it is required that no previous case has already violated the precedential constraint; the consequence of such an assumption is that no conflict of precedents can be found in deciding a new case. The consistency requirement can be unrealistic and not met by real legal case bases, as emphasized in \cite{Canavotto22}.  For this reason, in \cite{FlorioLLRS23}, we expanded the scope of the use of CMs with respect to \cite{liu2022modelling}  and have  abandoned the consistency assumption. To this extent, we introduced  a preference order among cases, whereby in the event of conflict between precedents the  preferred precedent should be followed. This solution seemed reasonable to us, since the preference order may result, for example, from the hierarchy of courts issuing the conflicting decisions and/or temporal dimension of cases, i.e., when the cases were decided.   The need to consider  these two elements, in order to model a precedential constraint that is more  adherent to legal practice (at least for common law systems) was already highlighted in  \cite{Broughton}. There, besides a horizontal precedential constraint, a tighter vertical precedential constraint, depending on court hierarchiy and based on the so called strict model is introduced.

In this paper, we pursue the intuition that legal case-based reasoners are nothing but binary classifiers, and we address the need to consider both the hierarchical and time element from the perspective of CMs.  The paper is structured as follows. Section 2 enriches the CMs with three components: 1) an organization, specifying the hierarchy of courts and which courts issue binding decisions with respect to other courts; 2) a preorder, which determines the temporal relation between cases; and 3)  a relevance relation among cases. On this basis, Section 3 formally defines the notion of precedent and binding precedent. The latter is a potentially constraining case for later decisions,  in the sense that it is  susceptible to two possible  exceptions: it can be  overruled by a later decision of a court which is entitled to do so, or it can by declared assessed \emph{per incuriam}. \emph{Per incuriam} cases are decided in the ignorance of a binding authority, i.e., in the ignorance of a binding precedent. In our framework we can define both notions of overruled and \emph{per incuriam} precedents.    
Taking into account only binding precedents without exceptions, Section 4 formalizes a temporal and hierarchical principle for handling conflicting precedents and verifies when this principle defines an unambiguous decision process.

%%%%%%%%%%%%%%%%%%%%%%%%%%%%%%%%%%%%%%%%%%%%%%%%%%%%%%%%%%%%%%%%%%%%%%%%
\section{Organization/Jurisdiction and Temporal Jurisdictional Classifier Models} \label{sec:basic}
In this section we instantiate the hierarchical and temporal dimensions  in the framework of CMs.  First, we define the notion of organisation, which characterises the  specific legal system. It specifies the courts, the hierarchical relation among them  and which courts issue binding decisions. As in \cite{Broughton}, we assume that court hierarchy  has a tree structure.

\begin{definition}[Organisation/jurisdiction]\label{def:organisation}
We call organisation/jurisdiction the triple $Org= (\Courts, \relHierarchy,\relBinding)$. Where: 1) $\Courts$ is a finite and non-empty set of courts;   2) $\relHierarchy,\relBinding\subseteq \Courts\times\Courts$; 3) $\relHierarchy$ is transitive and irreflexive\footnote{We adopt following definition  of irreflexivity: not $\court \relHierarchy \court$, for all  $\court\in\Courts$.  };  4) $\relHierarchy$ is ``tree like'', i.e. for all $\court_{i}, \court_{j}, \court_{k}\in \Courts$ if $ \court_{k} \relHierarchy \court_{i}$ and $\court_{j} \relHierarchy \court_{i}$  then $\court_{k}\relHierarchy \court_{j}$ or $\court_{j} \relHierarchy \court_{k}$ or $\court_{j}=\court_{k}$; 5) there is a root $c_0 \in \Courts$ such that for any $\court_{i}\in\Courts$, $\court_{0} \relHierarchy \court_{i}$.
\end{definition}

$ \court_{i} \relHierarchy \court_{j}$  is read as ``$\court_{i}$ is hierarchically higher than  $\court_{j}$'';  $ \court_{i} \relBinding\court_{j}$ is read as ``$\court_{i}$ issues binding decisions for  $\court_{j}$''. 
We did not specify the link between hierarchical and binding relations. Thus, different legal systems could be addressed.  However, in this work, we will focus  on \emph{common law systems}. Within common law systems, typically, 1) \emph{vertical stare decisis} applies, i.e. decisions of higher courts bind lower courts:
$ \text{ if } c_i\relHierarchy c_j \text{ then } c_i \relBinding c_j$;
2) lower courts have no binding power on higher courts and  two  not  hierarchically related courts don't issue binding decisions one for the other: $ \text{ if not } c_i\relHierarchy c_j \text{ then not } c_i \relBinding c_j$; 3) \emph{horizontal stare decisis} may apply, i.e. some $c_i$ may be self-bound, namely $c_i \relBinding c_i$. 
Ultimately, from 1), 2) and 3) follows \vspace{-8pt}
\begin{align*}\tag{SD}\label{eq:SD}
 \relHierarchy\subseteq \relBinding\subseteq \relHierarchy\cup \mathcal{I} \\[-25pt]
\end{align*}
where $\mathcal{I}$ is the identity relation on $\Courts$.

\begin{example}\label{Ex:Org}
We will take inspiration from the civil court system of England and Wales. In this system vertical stare decisis applies.  At the top of the courts hierarchy there is the Supreme Court of UK (which we will call $c_{0}$).  Officially, the Supreme Court is not bound by its own previous decisions, but it regards them as binding (although, the “Court may depart from them  when it appeared ‘right to do so'."\cite{Allermuir} We will return to this later). So, we  treat  the Supreme Court as self-bound.  Below the Supreme Court there is the Court of Appeal (which we will call $c_{1}$), which issues binding decisions for itself. Further down the hierarchy, there is the High Court ($c_{2}$). For simplicity, we will say that the High court is bound by its own  decisions.\footnote{Actually, some authorities suggest that High Court is self-bound  only if it acts as appellate court \cite{Graham}.}
Below the High court there are the County Courts  (about 170 \cite{CrimeAndCourtsAct}) that do not issue binding decisions. 
For simplicity, we will not consider all 170 County Courts but only two  of them ($c_{3}$ and $c_{4}$).
We model the courts system described  as follows.   We define the set of courts  $\Courts_{ex} =\{\court_{0}, \court_{1}, \court_{2}, \court_{3}, \court_{4}\}$ and the relation $\relHierarchy_{ex}\subseteq \Courts_{ex}\times\Courts_{ex}$, $\relHierarchy_{ex} = \{(c_i,c_j)\mid i<j, 0\leq i\leq 2, 1\leq j\leq 4\}$.
From what we said sofar, $c_{0}$, $c_{1}$ and $c_{2}$ are self-bound. Then, the binding relation is 
$\relBinding_{ex}=  \relHierarchy_{ex} \cup \{(\court_{0}, \court_{0}),(\court_{1}, \court_{1}),(\court_{2}, \court_{2}) \}$.   So, our organisation  is $Org_{ex}= (\Courts_{ex},\relHierarchy_{ex},\relBinding_{ex}).$
\end{example}
 The fundamental principle behind the doctine of \emph{stare decisis} is that a case at hand must be decided in the same way of a relevant and binding precedent.
However, decisions may sometimes be overruled. Overruling occurs when a court  decides  a case differently from a  relevant (binding) precedent \cite{interpreting, Rigoni_2014}. Not all courts  have overruling power. We will  assume that a higher court $c_i$ can always overrule a decision made by a lower court (such a precedent is not binding for $c_{i}$ but could be binding for another lower court). Moreover, we will assume that a court can overrule its own  precedents if certain conditions are met in the decision for the new case. The practice in the UK Supreme Court, for example, is to sit in a larger panel than usual when overruling its previous decisions \cite{MichaelRyle,interpreting}. 
Actually, we won't specify 
 which conditions must be met: in the description of each case we will simply add a parameter if in deciding the case the court can overrule its own precedents.  
 
We  introduce now the temporal and jurisdictional classifier models. 
We start by considering a \emph{finite} set of possible input values of the classifiers, $\atm_{0}$. 
In $\atm_{0}$ there are: variables describing the facts of the case, belonging to the finite set $\Facts$; the courts that can make the decision, belonging to the finite set $\Courts$; a parameter $l$ that, intuitively, means that the court  assessing the case at hand can overrule a previous decision made by the same court.  
So, $\atm_{0}= \Facts \cup \Courts \cup \{l\}$.
The classifier outputs have values  in $\val = \{1, 0, ?\}$ where elements stand for \emph{plaintiff wins}, \emph{defendant wins} and \emph{absence of decision} respectively. For $\outcome \in \{0, 1\}$, the ``opposite'' $\overline{\outcome}$ is noted for the value $1 - \outcome$.

\begin{definition}[Temp. Jur. Classifier Model]\label{def:classifiermodel}
    A temporal jurisdictional classifier model
    is defined as the tuple $C= (S, f, \Org, \temporalorder, \relevance)$. Where, $S\subseteq 2^{\atm_{0}} $ is a set of states, s.t. $\forall s \in S, \exists! \court \in \Courts$ s.t. $c\in s$; $f: S \longrightarrow \mathit{Val}$ is a  classification function; $ \Org$ is an organisation; $\temporalorder$ is a  total preorder on S \footnote{$\temporalorder$ is a transitive, reflexive relation such that forall $s, s'\in S$, $s\temporalorder s'$ or  $s'\temporalorder s$.} and  $\relevance\subseteq S\times S$ is a binary relation on $S$.     
\end{definition}

 Each state  of the model contains a factual situation together with a  unique court. In this sense, a state $s\in S$ represents a case presented to  a specific court $c\in\Courts$. In each state we can have an additional parameter $l$: if $l\in s$, we know that the court $c\in s$ can decide the case overruling a previous decision made by $c$ itself. 
 
 The classification function maps each state into a possible value, namely $\{ 0,1,?\}$. These values represent the decision associated to each case. So that, for each $s\in S$, we can have either that $f(s)=\outcome$,  with $\outcome\in\{0,1\}$ and so $s$ represents an assessed case; or we can have that $f(s)=?$, and so $s$  is an unassessed/new case.
 
 $Org$,  specifies the hierarchical and binding relations between the courts (see Def. \ref{def:organisation}). 
 
 $\temporalorder$ is a temporal preorder: $s\temporalorder s'$ is read as  “$s'$ was not assessed before $s$".   We  will say that $s$ is simultaneous to $s'$, noted $s=_{T}s'$,  iff $s\temporalorder s'$ and  $s'\temporalorder s$.  From  $\temporalorder$ we can retrieve the strict order $\stricttemporalorder$:  $ s\stricttemporalorder s'$ iff $ s \temporalorder s'$ and $s' \not \temporalorder s$. Then, $s \stricttemporalorder s'$ means “s was strictly decided before s'". 
 \color{black} We have chosen to use a preorder between the states, and not an order. In this way, we allow for two cases to be simultaneous. It could be  argued that, given two cases, it is always possible to determine which was decided first.  However, the temporal dimension  was mainly introduced to determine which precedents could be considered binding. We claim here that, given a case $s$, a relevant case $s'$ decided the same day or the same week cannot be considered binding: the court assessing $s$ would not have the possibility to  take $s'$ into account (just because, e.g., the reasons of $s'$ have not yet been released). In this sense we can consider two cases decided on the same day or in the same week as simultaneous, depending on the temporal granularity chosen. 
 
 $\relevance$ characterises the notion of \emph{relevance}: $s\relevance s'$ is read as ``$s$ is relevant for $s'$''. We  write $\relevance(s')= \{s\mid s \relevance s'\}$, for  the set of states that are relevant for $s'$. The relevance relation will be crucial in defining the notion of precedent. Roughly, a precedent for a case is a relevant previously assessed case. We don't impose any  specific notion of relevance on the model. Also, no specific property is required for the relevance relation. This allows us to take different notions of relevance into account. For example, we will show that  a relevance relation  can be defined  within the model, based on the  \emph{a fortiori reasoning }introduced in Horty's result model \cite{Horty2011RR}.
 
Henceforth, we will consider  \emph{complete} classifier models. These  models  take into account every possible combination of facts  and  courts; moreover, if a court $c$ is not self-bound, the classifier considers, among the states containing $c$, only those  containing also the parameter $l$ (since $c$ is not  self-bound, it can  overrule its own decisions). 
Finally, the temporal relation is such that new cases are always strictly after  already assessed cases. 

 \begin{definition}[Complete classifier]\label{def:completeClass}
 Let $\overline{2}^{\atm_{0}}\defin  \{s\in  2^{\atm_0}\mid l\in s$ if not $  c \relBinding c$, where $\court \in s \}$. Then, a  classifier model $C=(S, f, Org, \temporalorder, \relevance)$ is \emph{complete} iff $S = \overline{2}^{\atm_{0}}$ and $\forall s,s'\in S$ s.t. $f(s)\neq ?$ and $f(s')=?$, $ s<_{T}s'$.
\end{definition}

As mentioned above, we  can define a relevance relation   extracted  from the “a fortiori constraint" introduced  by Horty in the context of the  result model \cite{Horty2011RR}. 
We adapt notation  of \cite{liu2022modelling}. In the result model, the facts  are called factors. These  are “legally relevant fact patterns favouring one of the two opposing parties". Thus, we can have: factors for the plaintiff ($\plt$) and factors for the defendant ($\dfd$). So, $\Facts= \plt \uplus \dfd$.
We  use following notation: for $\outcome\in \{0,1\}$, $\Facts^{\outcome}= \plt$ iff $\outcome=1$,  $\Facts^{\outcome}= \dfd$  iff $\outcome=0$.

\begin{definition}[A fortiori-based relevance, $\relevance_{F}$] \label{def:hortyrelevance} 

Consider an assessed case, so  a state $s\in S$, such that $f(s)=\outcome$, $\outcome \in \{0,1\}$. Let $s'\in S$.  Then, based on a fortiori constraint,  $s$ is relevant for  $s'$,  namely $s\relevance_{F} s'$, iff $s \cap \Facts^{\outcome} \subseteq s'\cap\Facts^{\outcome} \text{ and }
s'\cap \Facts^{\overline{\outcome}} \subseteq s\cap\Facts^{\overline{\outcome}}$.
\end{definition}

Namely, suppose that $f(s)=1$. Then $s\relevance_{F} s'$, iff $s'$ has: a) at least the same factors in favour of $1$  that are in $s$, b) no more factors in favour of $0$ with respect to $s$.

\begin{example}[Running example]\label{runningEX1}
Consider $Org_{ex}$ of Example \ref{Ex:Org}. Let $Atm_{0}=\Facts\cup \{l\}\cup \Courts_{ex}$.  Suppose  $\Facts= \plt  \uplus  \dfd$, with $\plt=\{p,q,r\}$, $\dfd=\{t,v\}$.  
Let $s_{1}=\{p,t,v, c_{0}\}, s_{2}=\{p,r,t,v,c_{1}\} , s_{3}=\{p,r,t,v,l,c_{0}\},   s_{4}=\{p,t,v, c_{1}\}, s_{5}=\{p,t,q,c_{1}\}$. Let $f= \overline{2}^{\atm_{0}}\longrightarrow \{0,1, ?\}$ s.t $f(s_{1})=f(s_{2})=f(s_{4}) =1$, $f(s_{3})=f(s_{5})=0 $, and $f(s)=?$ otherwise. 
Consider the temporal preorder $\temporalorder$   on  $S$ s.t. 
 $s_{1}\stricttemporalorder s_{2} \stricttemporalorder s_{3}\stricttemporalorder s_ {4} \stricttemporalorder s_{5} \stricttemporalorder s'$, for any $s'$ s.t $f(s')=?$. 
 We  define the classifier model $C_{ex}= ( \overline{2}^{\atm_{0}}, f, \Org_{ex}, \temporalorder, \relevance_{F})$.  
 Consider the new  case $s^{*}=\{p,t, q,v, c_{2}\}$ to be decided by court $c_{2}$, i.e.  $f(s^{*})= ?$. It can be verified that, $\relevance_F(s^{*})= \{s_{1}, s_{4}, s_{5}\}$, $s_1, s_2\in \relevance_F(s_{s_3})$, $s_{3}\in  \relevance_F(s_{s_{4}})$, $s_{4}\in  \relevance_F(s_{s_{5}})$.
\end{example}

 $\relevance_{F}$  is neither reflexive, symmetric nor transitive. E.g. non reflexivity follows from the fact that, if $f(s)= ?$, then obviously $s\not\in \relevance_{F}(s)$, by definition. This is one reason why we have not imposed any properties on the relevance relation. But, $\relevance_{F}$  satisfies a variant of left-euclideanity that applies only to cases decided in opposite directions. 
\begin{proposition}\label{HortyLeft}
  Let $s, s', s''\in S$, $f(s)=\outcome\in\{0,1\}$. If  $f(s')=\overline{\outcome}$, $s\relevance_{F}s''$ and  $s' \relevance_{F} s''$, then $s\relevance_{F}s'$ and $s'\relevance_{F}s$.
\end{proposition}

\begin{proof}
 Let $s, s', s''\in S$, s.t. $f(s)=\outcome$, $f(s')=\overline{\outcome}$. Suppose $s\relevance_{F}s''$ and $s' \relevance_{F} s''$.  Then we have both 
\begin{align*}
&s \cap \Facts^{\outcome} \subseteq s''\cap\Facts^{\outcome} \text{ and }
s'' \cap \Facts^{\overline{\outcome}} \subseteq s\cap\Facts^{\overline{\outcome}}\\
&s '\cap \Facts^{\overline{\outcome}} \subseteq s^{''}\cap\Facts^{\overline{\outcome}} \text{ and }
s'' \cap \Facts^{\outcome} \subseteq s'\cap\Facts^{\outcome}
\end{align*}
Hence, 
\begin{align*}
&s '\cap \Facts^{\overline{\outcome}} \subseteq s\cap\Facts^{\overline{\outcome}} \text{ and } s \cap \Facts^{\outcome} \subseteq s'\cap\Facts^{\outcome}
\end{align*}

\noindent So  $s\relevance_{F}s'$ and $s'\relevance_{F}s$. 
\end{proof}

\section{Precedents, Binding Precedents, Overruled and \emph{Per Incuriam} Precedents} \label{sec:incuriam}
Based on the relevance relation  we  define  the notion of supporting precedent:  $s$ is a supporting precedent, or simply a precedent, for $s'$  if  $s$ is relevant for $s'$ and   decided before $s'$. In this section, for $s,s',s''\in S$, we  denote their courts $c,c',c''$,  (e.g. $c\in s\cap \Courts$). 
\begin{definition}[Supporting precedent]\label{def:support}
 Let $s,s'\in S$ and $\outcome\in \{0,1\}$. $s$ is a (supporting) precedent for $s'$ in the direction of $\outcome$, noted $\Pi(s,s',\outcome)$,  iff $f(s)=\outcome,  s \in \relevance( s' ) \text{ and } s<_{T}s'$.
\end{definition}

Henceforth, we will focus  on  binding precedents. Binding precedents for a state  $s'$ decided by  court  $\court'$ are those   precedents issued by a court $\court$ that has binding power on $\court'$.
\begin{definition}[Binding precedent]\label{def:bindingprec}
Let $s, s'\in S$ and $\outcome\in \{0,1\}$.
$s$ is a binding precedent for $s'$, for a decision as $\outcome$, noted $\beta (s,s',\outcome)$, iff $\Pi(s,s',\outcome) \text{ and } \relBinding(\court, \court')$. We simply write $\beta (s,s')$ iff there is $\outcome\in \{0,1\}$ s.t. $\beta (s,s',\outcome)$. 
\end{definition}
As discussed before, we are considering legal systems in which vertical stare decisis applies  and  horizontal stare decisis may apply. From (\ref{eq:SD}), we have the following.

\begin{proposition}\label{prop2}
Let $s, s'\in S$. If $\beta(s,s')$  then $\relHierarchy(\court, \court')$ or $c=c'$. 
\end{proposition}

\begin{proof}
    Since  $\beta(s,s')$  then  $\relBinding(\court, \court')$. By $(SD)$,  $\relHierarchy(\court, \court')$ or $c=c'$
\end{proof}

\begin{example} [Ex. \ref{runningEX1} continued]\label{runningEX2}
Recall  $\relevance_F(s^{*})= \{s_{1}, s_{4}, s_{5} \}$. For $s\in \relevance_F(s^{*})$,  $s<_{T}s^{*}$, so  every $s$ is a precedent for $s^{*}$ ($\Pi(s,s^{*},f(s)$), and is binding  for $s^{*}$, ($\beta(s,s^{*},f(s)$), since it assessed by a higher court than $s^{*}$. Also,  $s_1, s_2\in \relevance_F(s_{s_3})$ are  precedents for $s_{3}$, but only $s_1$ is binding  ($\beta(s_1, s_3)$) , ($s_2$ is decided by a lower court). Also, $\beta(s_3, s_4)$  and $\beta(s_4, s_5)$.  
\end{example}
As mentioned in the previous section, not all courts can overrule previous decisions. We assumed that $c'$ has the power to overrule a  decision made by $c$, in deciding a case $s'$ if: 1) $c'$ is hierarchically higher than $c$, or if, 2) $c'=c$ and in the decision of $s'$  the conditions for $c$ to overturn  its own precedents are met, namely if $l \in s$. With conditions 1) and 2) we establish that a lower court cannot overrule a higher court.\footnote{I.e. no anticipatory overruling is permitted. By “ anticipatory overruling of a precedent we refer to a lower court refusing to follow a precedent in anticipation of the likelihood that a higher court will overrule it" \cite{interpreting}.}

\begin{definition}[Overruling power]\label{def:overrulingpower}
Let $s'\in S$. Court $c'$ has the power to overrule (a decision  by) court $c$, when deciding $s'$, noted $O (c', c\mid s' )$, iff $\relHierarchy (c',c) \text{ or } (c'= c \text{ and }  l\in s)$.
\end{definition}
We can now define the notion of  overruled case.  Intuitively, a case $s'$ overrules a precedent $s$, when $s'$ is decided by $c'$ in the opposite direction wrt. to $s$ and  $c'$ has overruling power over $c$, when assessing $s'$ (i.e. $O (c', c\mid s' )$). 
\begin{definition}[Overruled case/state]\label{def: overruledcase} 
Let $s, s'\in S$. $s'$ overrules $s$, noted  $O(s',s)$, iff  \;$\Pi(s,s',\outcome)  , f(s')= \overline{\outcome}, \text{ and } O(c',c \mid  s').$
We write $Overruled_{T}(s,s')$, to mean that $s$ was overruled before $s'$ is assessed iff there is  $s''\in S$, s.t. $O(s'',s)$ and $s''<_{T} s'$.

\end{definition}
As consequence of Def. \ref{def:overrulingpower} a later case $s'$  overruled  $s$ if: it was decided by a higher court or it was decided by the same court and $l\in s'$. Hence, if a state $s'$ overrules a binding precedent $s$, necessarily  it is decided by the same court and $l\in s'$. 

\begin{proposition}
    Let $s,s'\in S$. If $\beta(s,s')$ and $O(s',s)$ then $c=c'$ and $l\in s'$. 
\end{proposition}

\begin{proof}
Suppose  $\beta(s,s')$ and $O(s',s)$. Then, by   $\beta(s,s')$ and  Prop.  \ref{prop2} it follows that  $\relHierarchy(c,c')$ or $c=c'$. From  $O(s',s)$, it follows that $\relHierarchy(c',c)$ or $c=c'$ and $l\in s'$. Hence it must be  $c=c'$ and $l\in s'$.

\end{proof}

 Overruled cases are no longer binding for subsequent decisions. In this, overruling constitutes an exception to binding precedents. We can therefore start filtering the binding cases for $s$, eliminating the cases that were overruled at the time $s$ was decided. In doing so we  also check whether $l\in s$: if a court can overrule itself in deciding $s$ (i.e. $l\in s$)  then it is not  bound by its own precedents. In this situation we will remove the  precedents issued from the same court from the set of binding precedents. 
 
The set of  binding precedents not overruled when $s$ was decided  is \vspace{-5pt}
\begin{equation*}
\tilde{\beta}_{s} =\begin{cases}
  \{s' \mid \beta(s',s) \text{ and not }  Overruled_{T}(s', s)\} &\text{ if } l\not\in  s;\\ 
 \{s' \mid  s\cap \Courts \neq s' \cap \Courts,  \beta(s',s) \text{ and not }  Overruled_{T}(s', s) \} &\text{ otherwise. } 
    \end{cases}\\[-5pt]
\end{equation*}
 Clearly, $s$ has no power to overrule any precedent in $\tilde{\beta}_{s}$.  
 \begin{proposition}
 If  $s'\in \tilde{\beta}_{s}$, then not $O(c,c'\mid s)$.    
 \end{proposition}

  \begin{proof}
    
Since $s'\in \tilde{\beta}_{s}$, then $\beta(s',s, f(s'))$ and so $\Pi(s',s, f(s'))$. Moreover, by Proposition  \ref{prop2},     $\relHierarchy(c',c)$ or $c=c'$. Suppose by contradiction that $O(c,c'\mid s)$, then either $\relHierarchy(c,c')$ or $c=c'$ and $l\in s$. The only possibility is that  $c=c'$ and $l\in s$. But then, by definition of $\tilde{\beta}_{s}$, $s'\not \in \tilde{\beta}_{s}$. Contradiction. 
\end{proof}

\begin{example}[Ex \ref{runningEX1} continued] \label{runningEX4}
$s_1$ and $s_2$ are precedents for $s_3$. $s_3$ overrules $s_2$, ($O(s_3, s_2)$), since $\relHierarchy(c_0,c_1)$,   $c_0\in s_3$  and $c_1\in s_2$ . $s_3$ overrules $s_1$:  $c_0\in s_1\cap s_3$  and  $l\in s_3$.  $s_3<_{T}s^{*}$, then $Overruled_{T}(s_1, s^{*})$.  Recall,  $s_{1}, s_{4}, s_{5}$ are binding for $s^{*}$. $s_{4}, s_{5}$ are not overruled ( $s_4$ is only precedent of $s_5$, but $l\not \in s_5$); then $\tilde{\beta}_{s^{*}}=\{s_4, s_5\}$. Also, $s_3$ is not overruled and  is binding for $s_4$, so $s_3\in \tilde{\beta}_{s_4}$.  $s_4\in \tilde{\beta}_{s_5}$  ($\beta(s_4,s_5)$, $s_4$ not overruled,  $c_1\in s_4\cap s_5$ and $l\not \in s_5$).
\end{example}

\noindent To some extent, we can say that an overruling case is a case that legitimately went  against a relevant precedent, in the sense that it was  decided in the opposite direction. Suppose now that there is $s'\in \tilde{\beta}_{s}$, s.t. $f(s')= \outcome\in\{0,1\}$ and $f(s)=\overline{\outcome}$. It would seem that $s$ went against a \emph{binding} precedent, illegitimately. The question is whether $s$ is to be considered  to be decided \emph{per incuriam}, that is in “ignorance of relevant binding authority" \cite{interpreting}. Broadly, if a precedent is decided \emph{per incuriam},  it loses its bindingness on later cases.
In modelling the notion of \emph{per incuriam}, we  take three interrelated aspects into account. 
\begin{varitemize}
\item[1)] Not all courts can   disregard a previous decision   taken \emph{per incuriam} \cite{LexisNexis}.
Specifically, we  assume that a lower court  cannot disregard a \emph{per incuriam}  precedent issued by a higher court, to which it nevertheless remains bound.\footnote{ In the UK case Cassell v Broome, it was affirmed that the Court of Appeal could not disregard a House of Lords decision, even though  \emph{per incuriam} \cite{interpreting}. In Canada, the same principle is usually adopted \cite{canada}.} Instead, we claim that a court may disregard a  \emph{per incuriam} binding precedent  if it was decided by the same court.\footnote{This is clearly stated for the UK  Court of Appeal in \emph{Young v. Bristol Aeroplane Co.} (1944).}  
\item[2)] We  infer from the model which cases are decided \emph{per incuriam}; such  information is not already given in the case base. Coming back to our question, suppose $s$  went  against a  $s'\in \tilde{\beta}_{s}$. Is this enough to say that $s$ was decided \emph{per incuriam}? It depends. 
If $s'$ was decided by a higher court, then $s$ was in any case \emph{per incuriam}, since $s$ could no disregard it.  Suppose instead $s$ and $s'$ were decided by the same court, then in deciding $s$, the court could have found that $s'$ went itself   against a binding precedent $s''\in \tilde{\beta}_{s'}$.   But to be sure that $s'$ was \emph{per incuriam}, and so not  actually binding for $s$, we have to check whether $s''$ was \emph{per incuriam} or not,  and so on. In this sense, for determining whether $s$ was \emph{per incuriam} we have to check each sequence of binding precedents that originates from $s$. 
\item[3)] The third aspect concerns against which precedents $s$ must have gone to be said decided \emph{per incuriam}. Namely, for a given case $s$ there could be conflicting binding  precedents, i.e.  precedents in $\tilde{\beta}_{s}$ according to which $s$ should be decided in opposite directions (both as 0 and as 1).   In such a situation, we argue that firstly the precedents \emph{per incuriam} in $\tilde{\beta}_{s}$ issued by the same court of $s$ must be ignored.  Among the remaining precedents, in line with a principle for the resolution of conflicts between precedents, to which we  return later, $s$ should have followed precedents decided later by the higher court. We informally refer to such precedents here as the best temporal and hierarchical binding precedents (bthbp) for $s$. 
Moreover, we highlight that, according to our temporal preorder, there may be simultaneous and conflicting bthbp for $s$. 
 In this case $s$ was in a situation of ‘genuine conflict’: any decision made by $s$ would have gone against a bthp. We cannot then say that $s$ was decided \emph{per incuriam}. Ultimately, $s$ can be said to have decided \emph{per incuriam} if it went  against a bthbp, in the absence of a ‘genuine conflict’. 
 \end{varitemize}
We want to capture all these insights in the definition of \emph{per incuriam}. To do so, we must first define the set of best hierarchical binding precedents in $\tilde{\beta}_{s}$. 

\begin{definition}
Let $\tilde{S}\subseteq S $, ${\sf Best}_H(\tilde{S})= \{s\in\tilde{S}\mid \not \exists s'\in \tilde{S} \text{ s.t } \relHierarchy(c',c)\}$.
 Let $s\in S$, the set of the  best  hierarchical precedents not overruled when $s$ was decided  is $Best_{H}(\tilde{\beta}_{s})$. 
 \end{definition}

To determine whether $s$ was decided \emph{per incuriam}, we first recursively generate the following graph $G=(V,E)$. At step 0 of the recursion, the graph ($G_{0}=(V_{0}, E_{0})$) has only one node $s$ ($V_{0}=\{s\}$) and no edge ($E_{0}=\emptyset$). At step 1, we add to the nodes  the states $s'$, that are the best hierarchical binding precedents for $s$ ($V_{1}= \{s\}\cup \{s\mid s'\in  Best_{H}(\tilde{\beta}_{s})\}$). Edges joining  $s$ to each $s'$ are added ($E_{1}=\{(s, s')\mid s'\in  Best_{H}(\tilde{\beta}_{s})\}$). In step 2, we reiterate the procedure  for each node added at step $1$.  The procedure is finite, the set of states $S$ being finite. 
Note  that a state is never a binding precedent for itself, so that there are no loops in the graph $G$. Now, we  compute whether $s$ is \emph{per incuriam}, by recursively exploring the graph $G$. We  claim that $s$ is \emph{per incuriam} if: 1) there is $s'$ ,  adjacent node to $s$ ($(s,s')\in E$), such that $s$ was decided in the opposite way ($f(s)\neq f(s')$), and such that if $s'$ was decided \emph{per incuriam} then it was decided by a higher court  than $s$ ($Incuriam (s') \Rightarrow  \relHierarchy(c',c)$); 2) there is no $s''$, adjacent node to $s$ ($(s,s'')\in E$), decided in the same direction as $s$ ($f(s)=f(s'')$),  at the same time or after $s'$ ($s'\leq_{T} s'' $), such that if $s''$ was decided \emph{per incuriam} then it was decided by a higher court ($Incuriam (s'') \Rightarrow  \relHierarchy(c'',c)$).

\begin{definition} [Per incuriam]
Let $s\in S$, with $f(s)=\outcome \in \{0,1\}$. Let $G= (V, E)= \bigcup_{n\geq 0} G_{n}$,  the  recursively defined oriented graph, where,  \vspace{-5pt}
\begin{itemize}
\item $G_{0}= (V_{0}, E_{0})$, where $V_{0}=\{s\}$ and $E_{0}= \emptyset$;

\item $G_{n+1}= (V_{n+1}, E_{n+1})$, where $V_{n+1}=V_{n}\cup \{s'\mid s'\in  Best_{H}(\tilde{\beta}_{s''}), \text{ with } s'' \in  V_{n}  \}$ and $E_{n+1}=E_{n}\cup \{(s'',s')\mid  s'\in  Best_{H}(\tilde{\beta}_{s''}), \text{ with } s'' \in  V_{n}\}.$\vspace{-5pt}
\end{itemize}
%where for every $\tilde{s}\in S$,  $B_{\tilde{s}}= \{s'\mid \beta(s',\tilde{s}) \text{ and not }  \exists s'' \text{ s.t. }  O(s'', s)  \text{ and } s''<_{T} s'$\}. 

 $s$ was decided per incuriam, noted $Incuriam (s)$, iff\vspace{-5pt}
\begin{itemize}
\item $\exists  s'$, s.t. $(s,s')\in E$ and $f(s)\neq f(s'),$ and ($Incuriam (s') \Rightarrow  \relHierarchy(c',c)$ ) and 
\item  $\nexists  s''$ s.t $(s,s'')\in E$ and $f(s)=f(s'')$, $s'\leq_{T} s''$ and ($Incuriam (s'') \Rightarrow  \relHierarchy(c'',c)$).
\end{itemize}
%where  $Ignores(s,s')$  iff $f(s)\neq f(s'), \text{ and not } O(c,c' \mid  s) $.
\end{definition}

\begin{example}[Ex. \ref{runningEX1} continued]
Recall $s_3\in \tilde{\beta}_{s_4}$, $f(s_{3})\neq f(s_{4})$.  $s_3$ was decided by $c_0$ a higher court than $c_1\in s_4$, so $Incuriam(s_{4})$. Recall $s_4\in \tilde{\beta}_{s_5}$.  $f(s_{4})\neq f(s_{5})$. But, $Incuriam(s_{4})$ and $s_4$ and ${s_5}$  are decided by the same court. Hence, $s_{5}$ is not per incuriam.
\end{example}

We can now define for each  $s$, the set of binding precedents without exception, $\beta_{s}$. We exclude  from the set  $\tilde{\beta}_{s}$ the precedents decided  \emph{per incuriam}  \emph{by the same court}. 

\begin{definition}[Binding precedents without exception]
 The set of  binding precedents without exceptions  for $s$ is  $\beta_{s}= \tilde{\beta}_{s}\setminus \{ s'\in \tilde{\beta}_{s}\mid  s'\cap \Courts = s \cap \Courts \text{ and }  Incuriam(s') \}$.
  \end{definition}

\begin{example} [Ex. \ref{runningEX1} continued]\label{ex:incuriam}
 Recall $\tilde{\beta}_{s^{*}}= \{s_4, s_5\}$ and  $Incuriam(s_4)$. $c_1 \in s_4$ is higher than $c_2\in s^{*}$. Then  $\beta_{s^{*}}=\tilde{\beta}_{s^{*}}$.  $s_4$ and $s_5$  have the same court. So, ${\sf Best}_H(\beta_{s^{*}})=\beta_{s^{*}}$. Note that  $\temporalorder$ is an order on  $\beta_{s^{*}}$, i.e.it is, besides transitive and reflexive, also  antisymmetric.
\end{example}

From the previous example we notice that in ${\sf Best}_H(\beta_{s})$ there may be  \emph{per incuriam} precedents decided by a higher court than  the one in $s$. Actually, we can verify that if the relevance relation is $\relevance_{F}$ (a fortiori reasoning based), the temporal relation is an order,  and in ${\sf Best}_H(\beta_{s})$ there are two cases decided in the opposite directions, then one of them is \emph{per incuriam}. The result is a consequence of the property in  Prop. \ref{HortyLeft}.

\begin{proposition}
    Let $C=(S,f,Org, <_{T}, \relevance_{F})$, a complete classifier, with $<_{T}$ a total order.  Let $s\in S$.  If $\exists s',s''\in {\sf Best}_H(\beta_{s})$, s.t. $f(s')\neq f(s'')$, then $Incuriam(s')$ or  $Incuriam(s'')$.
\end{proposition}

\begin{proof}
   \begin{itemize}
   \item[1)] Assume by contradiction that not $Incuriam(s')$ and not $Incuriam(s'')$.
   \item[2)] Without loss of generality assume $s'<_{T} s''$, by that $<_{T}$ is a total order.
   \item[3)] $s'\in \relevance_{F}(s'')$. Indeed $s'\in \relevance_{F}(s)$, $s''\in \relevance_{F}(s)$, $f(s')\neq f(s'')$. Then by Prop. \ref{HortyLeft}, $s'\in \relevance_{F}(s'')$.
   \item[4)] not $Overruled_{T}(s',s)$, since $s'\in  \beta_{s}$.
   \item [5)] since $s',s''\in {\sf Best}_H(\beta_{s})$,  $s'$ and $s''$ are assessed by the same court, namely, there is $\tilde{c}=s'\cap \Courts = s''\cap  \Courts$.
   \item [6) ]Since, $s', s''\in  {\sf Best}_H(\beta_{s})$, $s',s''<_{T} s$. By point 2)  $s<_{T}s'<_{T} s''$.
   \item [7)]$\Pi(s',s'', f(s'))$, this follows from $s'\in \relevance_{F}(s'')$(by point 3) and $s'<_{T} s''$ (point 2)).  
   \item [8)]$l\not\in s''$. Indeed suppose $l\in s''$. Then  $O(\tilde{c}, \tilde{c}\mid s'')$. So by point 5) $O(s'',s')$. But then since $s''<_{T} s$,  $Overruled_{T}(s',s)$. Contradiction, because of point 4). Then $l\not \in s''$.
   \item [9)] $\relBinding(\tilde{c},\tilde{c})$. This follows from  point 8) and definition of complete classifier.
   \item [10)] $s'\in \tilde{\beta_{s''}}$. Indeed, by point 7) and point 9), $\beta(s',s'', f(s'))$. Moreover $\tilde{c}=s'\cap \Courts = s''\cap  \Courts$, $l\not \in s''$. Also not $Overulled_{T}(s',s'')$ (since  not $Overulled_{T}(s',s)$  by point 4), and $s''<_{T}s$). Hence, $s'\in \tilde{\beta_{s''}}$. 
   \item [11)] We know that $s'\in \tilde{\beta_{s''}}$ and $f(s')\neq f(s'')$. By 1), not $Incuriam(s')$  and not $Incuriam(s'')$. Then by definition of per incuriam, there must be $s_{1}\in {\sf Best}_H(\beta_{s''}) $, s.t. $s<_{T} s_{1}$, $f(s')=f(s'')$ and $Incuriam(s_{1})\Rightarrow \relHierarchy (c_{1}, \tilde{c})$ (where $c_1\in s_{1}\cap \Courts)$. 
   \item [12)] $\Pi(s',s_{1}, f(s_{1}))$. Indeed $s'\in  \relevance_{F}(s'')$ and $s_{1}\in  \relevance_{F}(s'')$ (since $s_{1}\in \beta_{s''} $), $f(s_{1})= f(s'') \neq f(s')$. Then by Prop.   \ref{HortyLeft},   $s'\in  \relevance_{F}(s_{1})$. Moreover $s'<_{T} s_{1}$. So, $\Pi(s',s_{1}, f(s_{1}))$.
   \item [13) ]  $\tilde{c}= c_{1}$ and $l\not\in s_{1}$. Indeed, since $ s_{1}\in {\sf Best}_H(\beta_{s''}) $, then $\relHierarchy( c_{1}, \tilde{c})$ or $\tilde{c}= c_{1}$.  Suppose $\relHierarchy ( c_{1}, \tilde{c})$ or $\tilde{c}= c_{1}$ and $l\in s_{1}$. Then  $O(c_{1}, \tilde{c}\mid s_{1})$. So, by point 12), $O(s_{1},s')$. But then since $s_{1}<_{T} s$,  $Overruled_{T}(s',s)$. Contradiction, by point 4).  Then,  $\tilde{c}= c_{1}$ and $l\not\in s_{1}$.
     \item [14)] not $Incuriam(s_{1})$. Indeed by point 11) we know $Incuriam(s_{1})\Rightarrow \relHierarchy (c_{1}, \tilde{c})$ and by point 13) we know $\tilde{c}= c_{1}$. Hence  not $Incuriam(s_{1})$. 
     \item [15)] $s'\in \tilde{\beta}_{s_{1}}$. Indeed by points 9), 12)  13), $\beta(s',s_{1}, f(s'))$. Moreover $\tilde{c}=s'\cap \Courts = s_{1}\cap  \Courts$, $l\not \in s_{1}$ (point 13)). Also not $Overulled_{T}(s',s_{1})$ (since  not $Overulled_{T}(s',s)$   by point 4), and $s_{1}<_{T}s$). Hence, $s'\in \tilde{\beta}_{s_{1}}$.
     \item [16)]  We know that $s'\in \tilde{\beta}_{s_{1}}$ and $f(s')\neq f(s_{1})$. By 1) and 14),  we know not $Incuriam(s')$ and not $Incuriam(s_{1})$. Then, by definition of per incuriam, there must be $s_{2}\in {\sf Best}_H(\beta_{s_{1}}) $, s.t. $s'<_{T} s_{2}$ and $f(s_{1})=f(s_{2})$ and $Incuriam(s_{2})\Rightarrow \relHierarchy (c_{2}, \tilde{c})$ (where $c_2\in s_{2}\cap \Courts)$. 
     \item [17)] By reiterating the procedure in 12)-15) we obtain $c_{2}=\tilde{c}$, $l\not \in s_{2}$, $s'\in \tilde{\beta_{s_{2}}}$ and  $f(s_{2})\neq f(s')$.
     \item[18)] We can continue recursively, obtaining a sequence $s_{1}, s_{2}, s_{3}...s_{i}.. $ such that   not $Incuriam(s_{i-1})$ $\Rightarrow$  $\exists s_{i}$ s.t $s'<_{T}s_{i}<_{T}s_{i-1}<_{T}\cdots<_{T}s''<_{T}s$, $c_{i}=\tilde{c}$, $l\not\in s_{i}$, $s'\in  \tilde{\beta_{s_{i}}}$, $f(s_{i})\neq f(s')$ and not $Incuriam(s_{i})$. 
     \item [19)] Since the set of states $S$ is finite, we will find a $s_{n}$ in the sequence s.t. $s'<_{T}s_{n}$  and there is no $s_{n+1}$ in the sequence s.t $s'<_{T} s_{n+1}<s_{n}$. Hence $Incuriam(s_{n})$. But then $Incuriam (s_{n-1})$...$Incuriam (s_{1})$. Hence by point 11) $Incuriam(s'')$. Contradiction because of 1). Hence the thesis. 
   
    \end{itemize}
\end{proof}

\section{ A Legal Principle for Conflict Situations } \label{sec:principle}
In  the previous section, we stated that when there are conflicting binding precedents, the case at hand should follow the best temporal hierarchical binding precedents, namely the  the most recent cases decided by the highest court. The reason for this is that political, economic or social changes may affect a court's approach to a precedent \cite{interpreting}.  Such a principle is adopted, for example,  by courts in the United States \cite{interpreting, Broughton}. We refer to this principle as the  \emph{temporal hierarchical principle}.  In this section we formally define a decision making process for a new case $s^{*}$ ($f(s^{*})=?$) based on the temporal hierarchical principle. We verify when such a decision process is unambiguous.

\begin{definition} [Decision making process/function]
Let $s^{*}\in S$. A \emph{decision making process/function for $s^{*}$,  is any function $f^*: S \to 2^{\{0,1\}}$  defined as follows:\vspace{-5pt}
\begin{equation*}
f^{*}(s) =
        \begin{cases}
\{f(s)\}  &\text{ if } s\neq s^{*};\\ 
V \in 2^{\{0,1\}} &\text{ otherwise. }  
    \end{cases}\\[-10pt]
\end{equation*}}

\end{definition}
 So, a decision making process is a function that: to each state $s$ other than $s^{*}$ assigns the singleton given by the evaluation of the classification function in $s$; to $s^{*}$ assigns a subset $V$  of $\{0,1\}$(possibly empty). 
On the basis of the  cardinality of $V$  we determine whether the decision process   is unambiguous.  

\begin{definition}Let $f^*: S \to 2^{\{0,1\}}$ be a decision making process. We say that: 1) no decision can be made for $s^{*}$ iff  $f^{*}(s^{*})=\emptyset$; 2)   $f^{*}$   is unambiguous for $s^{*}$ iff $|f^{*}(s^{*})| = 1$; 3) $f^{*}$   is ambiguous for $s^{*}$ iff $|f^{*}(s^{*})|>1$.
\end{definition}

\noindent Thus, the decision making function $f^{*}$  is ambiguous for $s^{*}$ iff $f^{*}(s^{*})= \{0,1\} $, i.e. when  $s^{*}$ can be decided both as $0$ and as $1$ (thus, we have a conflict situation).
A first approach to set a decision making process could rely on binding precedents. Let $s^{*}$  be as in Ex. \ref{ex:incuriam}  and suppose to define a decision making function associating to $s^{*}$ the values ($0$ or $1$) assumed by  its binding precedents in  $\beta_{s^{*}}$, i.e. s.t.  $f^{*}(s^{*})= \{f(s') \mid s' \in \beta_{s^{*}} \}$. So, since $s_{4}, s_{5}\in \beta_{s^{*}}$,   $f^{*}(s^{*})=\{0,1\}$. We obtain an ambiguous decision making process for $s^{*}$. 
We   now formalise a decision making process based on the \emph{Temporal Hierarchical Principle}.

\begin{definition} The set of  latest states in  $\tilde{S} \subseteq S$ is $\sf Best_{T}(\tilde{S})= \{s\in\tilde{S} \mid \forall s'\in  \tilde{S}: s'\temporalorder s\}$. The set of  best temporal hierarchical states in $\tilde{S}$ is ${\sf Best_{TH}}(\tilde{S})= {\sf Best}_{T}({\sf Best}_H(\tilde{S}))$. 
\end{definition}

\begin{definition}[Temporal Hierarchical Principle]
The decision making function $f_{2}^*: S \to 2^{\{0,1\}}$, based on binding precedents, according to the temporal hierarchical principle is \vspace{-10pt}
    \begin{equation*}
f_{2}^{*}(s) =
        \begin{cases}
\{f(s)\}  &\text{ if } s\neq s^{*}\\ 
\{f(s')\mid s'\in Best_{TH}(\beta_{s^{*}}) \} &\text{ otherwise. }
    \end{cases}\\[-10pt]
    \end{equation*}
    \end{definition}
 $f_{2}^{*}$  assigns to $s^{*}$ the set of the outcomes of the elements in $Best_{TH}(\beta_{s^{*}})$, which are the  latest binding precedents for $s^{*}$, without exceptions,  decided by the higher court.

\begin{example} [Ex. \ref{runningEX1} continued]
Recall that $ Best_{H}(\beta_{s^{*}})= \{s_4, s_5\}$. Notice that $s_{4}<_{T}s_{5}$, hence $ Best_{TH}(\beta_{s^{*}})= \{ s_5\}$. So  $f_{2}^{*}(s^{*})=\{f(s_{5})\}=\{0\}$.
\end{example}

As hinted in the previous example,   $f_{2}^{*}$ is an unambiguous decision making for $s^{*}$ when $\temporalorder$ is an  order, namely also antisymmetic. Indeed, from antisymmetry it follows that cases cannot be  synchronous and there is a unique latest case decided by the higher court. 
Differently stated, if the temporal relation  between cases is a preorder, there may be synchronous conflicting states in $ Best_{H}(\beta_{s^{*}})$ and thus $f_{2}^{*}$ can be ambiguous. 
\begin{proposition}
Suppose $\beta_{s^{*}}\neq \emptyset$.  If 
 $\temporalorder$ is a  total order (at least on $ Best_{H}(\beta_{s^{*}})$) then
 $f_{2}^{*}$ is unambiguous decision making for $s^{*}$. 
    \end{proposition}

\begin{proof}
Suppose $\temporalorder$ is an order on $ Best_{H}(\beta_{s^{*}})$. Suppose $\beta_{s^{*}}\neq \emptyset$.  Suppose by contradiction that   $f_{2}^{*}$ is not unambiguous decision making for $s^{*}$. Then we can have two possibilities: 
\begin{itemize}
\item $f_{2}^{*}(s^{*})=\emptyset$; but since $\beta_{s^{*}}\neq \emptyset$ then $ Best_{TH}(\beta_{s^{*}})\neq \emptyset$. Namely, there is $s\in Best_{TH}(\beta_{s^{*}}) $ and so $f(s)\in  f_{2}^{*}(s^{*})$. Contradiction. 
\item $|f_{2}^{*}(s^{*})|>1$; then, there are $s, s'\in Best_{TH}(\beta_{s^{*}})$,  s.t. $f(s)\neq f(s')$ (and so $s\neq s'$).  Notice that since $s, s'\in Best_{TH}(\beta_{s^{*}})$, then   $s, s'\in Best_{H}(\beta_{s^{*}})$. Since $s\in Best_{TH}(\beta_{s^{*}})$, it must be $s\temporalorder s'$. Since $s\in Best_{TH}(\beta_{s^{*}})$ it  must be $s'\temporalorder s$. Contradiction since $\temporalorder$ is an order on $ Best_{H}(\beta_{s^{*}})$, so antisymmetric, and $s\neq_{T} s'$. 
\end{itemize}
\end{proof}

\section{Conclusion} \label{sec:conclusion}

This work further develops the connection, already identified in \cite{liu2022modelling}, between CBR models and classifiers models (CMs).  We have shown how  two founding elements of legal decision-making, namely   organization of courts and the temporal relationship between cases can be incorporated in CMs. We were then able to model the notion of precedent, binding precedent and binding precedents susceptible to exceptions: precedents decided \emph{per incuriam} and overruled ones. Finally,  we introduced a temporal and hierarchical principle for handling conflicting precedents and verified when this principle determines an unambiguous decision process. 
Some aspects deserve further discussion. 

We first briefly discuss the relationship between our work and \cite{Broughton}, which had already highlighted the need to consider temporal and hierarchical dimensions in precedent constraint models. As already amply emphasised, our work moves within the framework of CMs. In \cite{Broughton}  two different models of precedential constraint are combined, starting from the assumption that  the vertical precedent constraint is much tighter than the horizontal one; in this sense, even if not explicitly stated, the relevance relation between cases  depends on which constraint applies. We don't make such a distinction. More specifically, we do not impose any specific notion of relevance between cases, in this sense our work moves on a more abstract level. For us, the vertical constraint is stronger only in the approach to  \emph{per incuriam}, when we argue that  a  \emph{per incuriam} precedent cannot be ignored by a lower court. 
A further distinction pertains to  how violations of the vertical constraint are handled: in \cite{Broughton}, if a case  in the case base violated the vertical constraint, then a decision for a new case cannot be forced (as, in \cite{Horty2011RR}, the precedential constraint cannot force a new case if the starting case base is not consistent); for us a case decided ignoring a vertical precedent is \emph{per incuriam}  and can  simply be discarded. 
Also, differently from \cite{Broughton}, we do not assume that horizontal constraint always applies,  this depends on the binding relation   associated to the CM.
Finally, in \cite{Broughton} the temporal relation between the cases is a total order, for us  is a total preorder; indeed,  the temporal and hierarchical principle, also discussed in \cite{Broughton}, is not always sufficient for us to solve conflicts. 

A further discussion deserve the  relevance and binding relations. Relevance  is a “binary" notion: a case is either relevant to another or it is not, and if it is relevant it cannot be distinguished. In future work we aim to refine the relevance notion so that  some cases can be more relevant than others. Furthermore, we did not impose any properties on the relevance relation, so that two cases may be relevant for a new case but not for each other.  This will be important to develop a  framework where multiple issues or intermediate factors have to be addressed to decide a case (see, among others \cite{branting,capon21,CanovottoHierarchies,VanWoerkomHieararchy}), so that two cases may be relevant to a new case but not to each other because they concern different issues/intermediate factors.
Also, the notion of binding is  “binary": a court  issues  binding decisions or not. But, sometimes,  as suggested in  \cite{deonticAuth}, a binding relation   between two agents of an institution may be  context-dependent. It may happen, as is the case with the High Court in England, that a court binds only in specific situations.  In the future, we plan to consider a more sophisticated binding relation. 

Last but not least, in  \cite{LiuLoriniJLC} CMs were defined as models for the modal logic BCL. Our goal is to develop an extension of the BCL logic and to provide a completeness result with respect to the class of temporal and jurisdictional CMs. 

\section*{Acknowledgements}
 Huimin Dong is supported by the Chinese Ministry of Education project of the Key Research Institute of Humanities and Social Sciences at Universities (22JJD720021).
Antonino Rotolo was partially supported by the projects CN1 “National Centre for HPC, Big Data and Quantum Computing” (CUP: J33C22001170001) and PE01
“Future Artificial Intelligence Research” FAIR (CUP: J33C22002830006). 
 
\bibliographystyle{plain}
\bibliography{biblio}

\end{document}